\documentclass{article}

\usepackage[colorinlistoftodos]{todonotes}

\usepackage[final]{neurips_2020}

\usepackage[utf8]{inputenc} % allow utf-8 input
\usepackage[T1]{fontenc}    % use 8-bit T1 fonts
\usepackage{hyperref}       % hyperlinks
\usepackage{url}            % simple URL typesetting
\usepackage{booktabs}       % professional-quality tables
\usepackage{amsfonts}       % blackboard math symbols
\usepackage{nicefrac}       % compact symbols for 1/2, etc.
\usepackage{microtype}      % microtypography

\usepackage{amsfonts,amssymb,amsmath,amsthm,mathtools} % bbm
\usepackage{graphicx,microtype,latexsym,lpic,bm,xspace,booktabs,wrapfig,array}

\usepackage[bf]{caption}

\usepackage[linesnumbered,ruled,vlined]{algorithm2e}
\usepackage{thmtools,thm-restate}
\usepackage{tocloft}

\usepackage{dsfont}
\usepackage[greek,english]{babel}
\addto\extrasenglish{}
\addto\extrasenglish{}
\usepackage{bbm}
\usepackage{amssymb}

\usepackage{enumitem}  %Avoid \usepackage{enumerate} ?
\setlist[itemize]{noitemsep,label=$-$}
\setlist[enumerate]{noitemsep}

\newtheorem{theorem}{Theorem}[section]
\newtheorem{definition}[theorem]{Definition}
\newtheorem{example}[theorem]{Example}

\newtheorem{lem}[theorem]{Lemma}

\newtheorem{claim}[theorem]{Claim}

\newcommand{\myignore}[1]{}

\newcommand{\TealBlue}[1]{\textcolor{TealBlue}{#1}}
\newcommand{\Peach}[1]{\textcolor{Peach}{#1}}
\newcommand{\Cyan}[1]{\textcolor{cyan}{#1}}
\newcommand{\Red}[1]{\textcolor{red}{#1}}
\newcommand{\Navy}[1]{\textcolor{Blue}{#1}}
\newcommand{\Blue}[1]{\textcolor{Blue}{#1}}
\newcommand{\Green}[1]{\textcolor{OliveGreen}{#1}}
\newcommand{\Black}[1]{\textcolor{black}{#1}}
\newcommand{\White}[1]{\textcolor{white}{#1}}
\newcommand{\Code}[1]{\Cyan{\texttt{#1}}}
\newcommand{\Gray}[1]{\textcolor{gray}{#1}}
\newcommand{\Magenta}[1]{\textcolor{magenta}{#1}}
\newcommand{\Maroon}[1]{\textcolor{Maroon}{#1}}

\newcommand{\tTealBlue}[1]{\ifnum\tdotoggle=1\TealBlue{#1}\else{#1}\fi}
\newcommand{\tPeach}[1]{\ifnum\tdotoggle=1\Peach{#1}\else{#1}\fi}
\newcommand{\tCyan}[1]{\ifnum\tdotoggle=1\Cyan{#1}\else{#1}\fi}
\newcommand{\tRed}[1]{\ifnum\tdotoggle=1\Red{#1}\else{#1}\fi}
\newcommand{\tNavy}[1]{\ifnum\tdotoggle=1\Navy{#1}\else{#1}\fi}
\newcommand{\tBlue}[1]{\ifnum\tdotoggle=1\Blue{#1}\else{#1}\fi}
\newcommand{\tGreen}[1]{\ifnum\tdotoggle=1\Green{#1}\else{#1}\fi}
\newcommand{\tBlack}[1]{\ifnum\tdotoggle=1\Black{#1}\else{#1}\fi}
\newcommand{\tWhite}[1]{\ifnum\tdotoggle=1\White{#1}\else{#1}\fi}
\newcommand{\tCode}[1]{\ifnum\tdotoggle=1\Code{#1}\else{#1}\fi}
\newcommand{\tGray}[1]{\ifnum\tdotoggle=1\Gray{#1}\else{#1}\fi}
\newcommand{\tMagenta}[1]{\ifnum\tdotoggle=1\Magenta{#1}\else{#1}\fi}
\newcommand{\tMaroon}[1]{\ifnum\tdotoggle=1\Maroon{#1}\else{#1}\fi}

\let\OLDthebibliography\thebibliography
\renewcommand\thebibliography[1]{
	\OLDthebibliography{#1}
	\setlength{\parskip}{1.3pt}
}

\newcommand{\inbrace}[1]{\left \{ #1 \right \}}
\newcommand{\inparen}[1]{\left ( #1 \right )}
\newcommand{\insquare}[1]{\left [ #1 \right ]}

\newcommand{\inabs}[1]{\begin{vmatrix} #1 \end{vmatrix}}

\newcommand{\set}[1]{\inbrace{#1}}

\DeclareMathOperator*{\sign}{sign}
\DeclareMathOperator*{\argmin}{argmin}

\DeclareMathOperator*{\Ex}{\mathbb{E}}

\DeclareMathOperator*{\Prob}{Pr}

\newcommand{\eps}{\varepsilon}

\newcommand{\bbN}{{\mathbb N}}

\newcommand{\bbR}{{\mathbb R}}

\let\boldm\bm
\renewcommand{\bm}{{\boldm m}}

\newcommand{\calA}{\mathcal{A}}
\newcommand{\calB}{\mathcal{B}}
\newcommand{\calC}{\mathcal{C}}
\newcommand{\calD}{\mathcal{D}}

\newcommand{\calG}{\mathcal{G}}
\newcommand{\calH}{\mathcal{H}}

\newcommand{\calU}{\mathcal{U}}

\newcommand{\calX}{\mathcal{X}}
\newcommand{\calY}{\mathcal{Y}}

\newcommand{\ERM}{{\rm ERM}}
\newcommand{\RERM}{{\rm RERM}}
\newcommand{\ind}{\mathbbm{1}}
\newcommand{\Risk}{{\rm R}}
\newcommand{\err}{{\rm err}}
\newcommand{\vc}{{\rm vc}}
\newcommand{\MAJ}{{\rm MAJ}}
\newcommand{\co}{{\rm co}}
\newcommand{\im}{{\rm im}}
\newcommand{\removed}[1]{}

\newcommand{\abs}[1]{\left\lvert #1 \right\rvert}

\colorlet{sgreen}{black!35!green}

\usepackage{marginnote}

\newcommand{\todol}[2][]{{%
 \let\marginpar\marginnote
 \reversemarginpar
 \renewcommand{\baselinestretch}{0.8}%
 \todo[#1]{#2}}}

\newtheorem{innercustomthm}{Theorem}
\newenvironment{customthm}[1]
  {\renewcommand\theinnercustomthm{#1}\innercustomthm}
  {\endinnercustomthm}

\usepackage{tcolorbox}
\usepackage{pdfpages}

\title{Reducing Adversarially Robust Learning to Non-Robust PAC Learning}

\author{%
  Omar Montasser\\
  \texttt{omar@ttic.edu}\\
  \And
  Steve Hanneke\\
  \texttt{steve.hanneke@gmail.com}\\ 
  \And
  Nathan Srebro\\
  \texttt{nati@ttic.edu}%\\
  \AND
  
  {\normalfont Toyota Technological Institute at Chicago}\\
}

\begin{document}

\maketitle

\begin{abstract}
  We study the problem of reducing adversarially robust learning to standard PAC learning, i.e. the complexity of learning adversarially robust predictors using access to only a black-box non-robust learner. We give a reduction that can robustly learn any hypothesis class $\calC$ using any non-robust learner $\calA$ for $\calC$.  The number of calls to $\calA$ depends logarithmically on the number of allowed adversarial perturbations per example, and we give a lower bound showing this is unavoidable.
\end{abstract}

\section{Introduction}

We consider the problem of learning predictors that are {\em robust} to adversarial examples at test time. That is, we would like to be robust against an adversary $\calU:\calX \to 2^{\calX}$ that can perturb examples at test-time, where $\calU(x)\subseteq \calX$ is the set of allowed corruptions the adversary might replace $x$ with, as measured by the {\em robust risk}:
\begin{equation}
    \label{eqn:rob-risk}
    \Risk_{\calU}(\hat{h};\calD) \triangleq \Ex_{(x,y) \sim \calD}\!\left[ \sup\limits_{z\in\calU(x)} \ind[\hat{h}(z)\neq y] \right].
\end{equation}
For example, $\calU$ could be perturbations of bounded $\ell_p$-norms \citep{DBLP:journals/corr/GoodfellowSS14}. 

We ask whether we can adversarialy robustly learn a given target hypothesis class $\calC \subseteq \calY^{\calX}$ (e.g. neural networks)---that is, whether, if there exists a predictor in $\calC$ with zero robust risk w.r.t.~some unknown distribution $\calD$ over $\calX\times \calY$, can we find a predictor with (arbitrarily small) robust risk using $m$ i.i.d.\ (uncorrupted) samples $S=\set{(x_i,y_i)}_{i=1}^{m}$ from $\calD$. Recently, \cite{pmlr-v99-montasser19a} showed that if $\calC$ is PAC-learnable non-robustly, then $\calC$ is also adversarially robustly learnable. However, their result is not constructive and the robust learning algorithm given is inefficient, complex, and does not actually directly use a non-robust learner.  In this paper, we ask a more constructive version of this question:
\begin{center}
\textit{Can we learn adversarially robust predictors given only black-box access to a non-robust learner?}
\end{center}

That is, we are asking whether it is possible to reduce adversarially robust learning to standard non-robust learning. Since we have a plethora of algorithms devised for standard non-robust learning, it would be useful if we could design efficient \textit{reduction} algorithms that leverage such non-robust learning algorithms in a black-box manner to learn \textit{robustly}. That is, design generic wrapper methods that take as input a learning algorithm $\calA$ and a specification of the adversary $\calU$, and robustly learn by calling $\calA$. Many systems in practice perform standard learning but with no robustness guarantees, and therefore, it would be beneficial to provide wrapper procedures that can guarantee adversarial robustness in a black-box manner without needing to modify current learning systems internally. 

\textbf{Related Work} Recent work \citep{DBLP:conf/soda/MansourRT15, DBLP:conf/colt/FeigeMS15, DBLP:conf/alt/FeigeMS18, DBLP:conf/alt/AttiasKM19} can be interpreted as giving reduction algorithms for adversarially robust learning. Specifically, \cite{DBLP:conf/colt/FeigeMS15} gave a reduction algorithm that can robustly learn a {\em finite} hypothesis class $\calC$ using black-box access to an $\ERM$ for $\calC$. Later, \cite{DBLP:conf/alt/AttiasKM19} improved this to handle {\em infinite} hypothesis classes $\calC$. But their complexity and the number of calls to $\ERM$ depend super-linearly on the number of possible perturbations $\left\lvert\calU\right\rvert = \sup_x \left\lvert\calU(x)\right\rvert$, which is undesirable for most types of perturbations---we completely avoid a sample complexity dependence on $|\calU|$, and reduce the oracle complexity to at most a poly-logarithmic dependence.  Furthermore, their work assumes access specifically to an $\ERM$ procedure, which is a very specific type of learner, while we only require access to any method that PAC-learns $\calC$ and whose image has bounded VC-dimension.

A related goal was explored by \citet{salman2020black}: They proposed a method to {\em robustify pre-trained predictors}. Their method takes as input a black-box \textit{predictor} (not a learning algorithm) and a point $x$, and outputs a label prediction $y$ for $x$ and a radius $r$ such that the label $y$ is robust to $\ell_2$ perturbations of radius $r$. But this doesn't guarantee that the predictions $y$ are correct, nor that the radius $r$ would be what we desire, and even if the predictor was returned by a learning algorithm and has a very small non-robust error, we do not end up with any gurantee on the robust risk of the robustified predictor.  In this paper, we require black-box access to a \textit{learning algorithm} (not just to a single predictor), but we output a predictor that {\em is} guaranteed to have \textit{small} robust risk (if one exists in the class, see Definition \ref{def:robust-pac}). We also provide a general treatment for arbitrary adversaries $\calU$, not just $\ell_p$ perturbations.%\natinote{Do they treat only $\ell_2$ or $\ell_p$ in general?}.\omarnote{They say that their method works on any $\ell_p$ norm as long as randomized smoothing works for any $\ell_p$ norm, and they cite some work on that. They do experiments only on $\ell_2$, but I am giving them benefit of the doubt.}

Finally, we note that the approach of \citet{pmlr-v99-montasser19a} can be interpreted as using black-box access to an oracle $\RERM_\calC$ minimizing the robust {\em empirical} risk:%\natinote{Is this displayed equation necesary or can we remove it?}\omarnote{I mention RERM a few times in Section 3 and I think it would be good to define it}
\begin{equation}
\label{eqn:rerm}
    \hat{h}\in \RERM_{\calC}(S) \triangleq \argmin_{h\in \calC} \frac{1}{m} \sum_{i=1}^{m} \sup\limits_{z\in\calU(x)} \ind[h(z)\neq y].
\end{equation}
But this goes well beyond just a {\em non-robust} learning algorithm, or even \ERM.

\textbf{Efficient Reductions} From a computational perspective, the relationship between standard non-robust learning and adversarially robust learning is not well-understood. It is natural to wonder whether there is a general efficient reduction for adversarially robust learning, using only non-robust learners. Recent work has provided strong evidence that this is not the case in general. Specifically, \cite{bubeck2019adversarial} showed that there exists a learning problem  that can be learned efficiently non-robustly, but is computationally intractable to learn robustly (under plausible 
complexity-theoretic assumptions). In this paper, we aim to understand when such efficient reductions {\em are} possible.

\textbf{Main Results} When studying reductions of adversarially robust learning to non-robust learning, an important aspect emerges regarding the form of access that the reduction algorithm has to the adversary $\calU$. How should we model access to the sets of adversarial perturbations represented by $\calU$? 

In Section~\ref{sec:explicit}, we study the setting where the reduction algorithm has explicit access/knowledge of the possible adversarial perturbations induced by the adversary $\calU$ on the \textit{training examples}. We first show that there is an algorithm that can learn adversarially robust predictors with black-box oracle access to a non-robust algorithm:

\begin{customthm}{3.1}[Informal]
For any adversary $\calU$, Algorithm~\ref{alg:robustify} robustly learns any target class $\calC$ using any black-box non-robust PAC learner $\calA$ for $\calC$, with $O(\log^2 |\calU|)$ oracle calls to $\calA$ and sample complexity independent of $|\calU|$. 
\end{customthm}

The oracle complexity dependence on $\left\lvert\calU\right\rvert$, even if only logarithmic, might be disappointing, but we show it is unavoidable:

\begin{customthm}{3.2} [Informal]
There exists an adversary $\calU$ such that for any reduction algorithm $\calB$, there exists a target class $\calC$ and a PAC learner $\calA$ for $\calC$ such that $\Omega(\log |\calU|)$ oracle queries to $\calA$ are necessary to robustly learn $\calC$. 
\end{customthm}

This tells us that only requiring a non-robust PAC learner $\calA$ is not enough to avoid the $\log|\calU|$ dependence, even with explicit knowledge of $\calU$. In Section~\ref{sec:mistake-oracle}, we show that having an {\em online} learner $\calA$ for $\calC$, allows us to robustly learn $\calC$ with access to a mistake oracle for $\calU$ (see Definition~\ref{def:mistake-oracle}) where no explicit knowledge of $\calU$ is assumed and no dependence on $|\calU|$ is incurred:

\begin{customthm}{4.2} [Informal]
There exists an algorithm $\calB$ that can robustly learn any target class $\calC$ w.r.t. any adversary $\calU$ when given access to a mistake oracle $\mathsf{O}_{\calU}$ and a black-box online learner $\calA$ for $\calC$. The sample complexity, number of calls to $\calA$, and number of calls to $\mathsf{O}_{\calU}$ are independent of $|\calU|$.
\end{customthm}

\section{Preliminaries}

Let $\calX$ denote the instance space and $\calY=\set{\pm1}$ denote the label space. Let $\calU:\calX\to 2^{\calX}$ denote an arbitrary adversary. For any adversary $\calU$, denote by $\left\lvert\calU\right\rvert \triangleq \sup_x \left\lvert\calU(x)\right\rvert$ the number of allowed adversarial perturbations. We start with formalizing the notions of non-robust (standard) PAC learning and robust PAC learning: 

\begin{definition} [PAC Learnability]
\label{def:weaklearn}
A target hypothesis class $\calC \subseteq \calY^\calX$ is said to be PAC learnable if there exists a learning algorithm $\calA:(\calX\times \calY)^*\to \calY^\calX$ with sample complexity $m(\eps, \delta):(0,1)^2\to \bbN$ such that: for any $\eps, \delta \in (0,1)$, for any distribution $\calD$ over $\calX\times \calY$, and any target concept $c\in\calC$ with zero risk, $\err_\calD(c)=0$, with probability at least $1-\delta$ over $S\sim \calD^{m(\eps,\delta)}$, 
\[\err_{\calD}(\calA(S)) \triangleq \Prob_{(x,y)\sim \calD} \insquare{\calA(S)(x) \neq y} \leq \eps.\]
\end{definition}

\begin{definition} [Robust PAC Learnability]
\label{def:robust-pac}
A target hypothesis class $\calC \subseteq \calY^\calX$ is said to be {\em robustly} PAC learnable with respect to adversary $\calU$ if there exists a learning algorithm $\calB:(\calX\times \calY)^*\to \calY^\calX$ with sample complexity $m(\eps,\delta):(0,1)^2\to \bbN$ such that: for any $\eps,\delta \in (0,1)$, for any distribution $\calD$ over $\calX\times \calY$, and any target concept $c\in\calC$ with zero robust risk, $\Risk_\calU(c;\calD)=0$, with probability at least $1-\delta$ over $S\sim \calD^{m(\eps,\delta)}$,
\[\Risk_{\calU}(\calB(S);\calD)\leq \epsilon.\]
\end{definition}

We recall the Vapnik-Chervonenkis dimension (VC dimension) is defined as follows:

\begin{definition}[VC dimension]
\label{VCdim}
We say that a sequence $\{x_1,\dots,x_k\}\in\calX$ is shattered by $\calC$ if $\forall y_1,\dots,y_k\in \calY, \exists h\in \calC$ such that $\forall i\in[k], h(x_i)=y_i$. The VC dimension of $\calC$ (denoted $\vc(\calC)$) is then defined as the largest integer $k$ for which there exists $\{x_1,\dots,x_k\}\in \calX$ that is shattered by $\calC$. If no such $k$ exists, then $\vc(\calC)$ is said to be infinite.
\end{definition}

Another important complexity measure that is utilized in the study of robust PAC learning is the notion of {\em dual} VC dimension, which we define below: 

\begin{definition}[Dual VC dimension]
\label{def:DualVCdim}
Consider a \emph{dual space} $\calG$: a set of functions $g_{x} : \calC \to \calY$ defined as $g_{x}(h) = h(x)$, 
for each $h \in \calC$ and each $x \in \calX$. Then, the dual VC dimesion of $\calC$ (denoted $\vc^*(\calC)$) is defined as the VC dimension of $\calG$. In other words, $\vc^*(\calC)=\vc(\calG)$ and it represents the largest set $\set{h_1,\dots, h_k}$ that is shattered by points in $\calX$.
\end{definition}
If the VC dimension is finite, then so is the dual VC dimesion, and it can be bounded as $\vc^{*}(\calC) < 2^{\vc(\calC)+1}$ \citep{assouad:83}.  Although this exponential dependence is tight for some classes, for many natural classes, such as linear predictors and some neural networks (see, e.g.\ Lemma~\ref{lem:dualvc-convexhull}), the primal and dual VC dimensions are equal, or at least polynomially related.

We also formally define what we mean by a \textit{reduction} algorithm:
\begin{definition} [Reduction Algorithm]
\label{def:reduction}
For an adversary $\calU$, a reduction algorithm $\calB_\calU$ takes as input a black-box learning algorithm $\calA$ and a training set $S\subseteq \calX \times \calY$, and can use $\calA$ by calling it $T$ times on inputs $\calB_\calU$ constructs each of size $m_0 \in \bbN$, and outputs a predictor $f\in\calY^{\calX}$. 
\end{definition}
We emphasize that $\calB_\calU$ is allowed to be adaptive in its calls to $\calA$.  That is, it can call $\calA$ on one constructed data set, then construct another data set depending on the returned predictor, and call $\calA$ on this new data set.  Such adaptive use of the base learner $\calA$ is central to boosting-type constructions.

We know that a hypothesis class $\calC$ is PAC learnable if and only if its VC dimension is finite \citep{vapnik:71,vapnik:74,blumer:89,ehrenfeucht:89}. And in this case, $\calC$ is properly PAC learnable with $\ERM_{\calC}$. \citet[Theorem 4]{pmlr-v99-montasser19a} showed that if $\calC$ is PAC learnable, then $\calC$ is adversarially robustly PAC learnable with an improper learning rule that required a $\RERM_\calC$ oracle (see Equation~\ref{eqn:rerm}) and sample complexity of $\Tilde{O}\inparen{\frac{\vc(\calC)\vc^*(\calC)}{\eps}}$. In this paper, we study whether it is possible to adversarially robustly PAC learn $\calC$ using only a black-nox non-robust PAC learner $\calA$ for $\calC$.  We will not require $\calA$ is ``proper''' (i.e.~returns a predictor in $\calC$), but we will rely on it returning a predictor in some, possibly much larger, class which still has finite VC-dimension.  To this end, we denote by $\vc(\calA)=\vc(\im(\calA))$ and $\vc^*(\calA)=\vc^*(\im(\calA))$ the primal and dual VC dimension of the image of $\calA$, i.e.~the class $\im(\calA) = \left\{ \calA(S) \middle| S \in (\calX\times\calY)^* \right\}$ of the possible hypothesis $\calA$ might return.  For $\ERM$, or any other proper learner, $\im(\calA)\subseteq\calC$ and so $\vc(\calA)\leq\vc(\calC)$ and $\vc^*(\calA)\leq\vc^*(\calC)$.

\section{Learning with Explicitly Specified Adversarial Perturbations}
\label{sec:explicit}

When studying reductions of adversarially robust PAC learning to non-robust PAC learning, an important aspect emerges regarding the form of access that the reduction algorithm has to the adversary $\calU$. How should we model access to the sets of adversarial perturbations represented by $\calU$? 

In this section, we explore the setting where the reduction algorithm has explicit knowledge of the adversary $\calU$. That is, the reduction algorithm knows the set of possible adversarial perturbations for each example in the training set. This is in accordance with what is typically considered in practice, where the adversary $\calU$ (e.g.\ $\ell_\infty$ perturbations) is known to the algorithm, and this knowledge is used in adversarial training (see e.g.\ \cite{DBLP:conf/iclr/MadryMSTV18}). Formally, we consider the following question:

\begin{center}
For any adversary $\calU$, does there exist an algorithm that can learn a target class $\calC$ {\em robustly} w.r.t $\calU$ given only a black-box non-robust PAC learner $\calA$ for $\calC$?
\end{center}

We give a positive answer to this question. In Theorem~\ref{thm:upperbound}, we present an algorithm (see Algorithm~\ref{alg:robustify})---based on the $\alpha$-Boost algorithm \cite[Section 6.4.2]{schapire:12} and recent work of \citet[Theorem 4]{pmlr-v99-montasser19a}--- that can adversarially robustly PAC learn a target class $\calC$ with only black-box oracle access to a PAC learner $\calA$ for $\calC$.

\begin{algorithm}[H]
\caption{Robustify The Non-Robust}\label{alg:robustify}
\SetKwInput{KwInput}{Input}                % Set the Input
\SetKwInput{KwOutput}{Output}              % set the Output
\SetKwFunction{FMain}{ZeroRobustLoss}
\SetKwFunction{AlphaBoost}{$\alpha$-Boost}
\DontPrintSemicolon
    \KwInput{Training dataset $S=\set{(x_1,y_1),\dots, (x_m,y_m)}$, black-box non-robust learner $\calA$}
  \BlankLine
  %\KwData{Testing set $x$}
  Inflate dataset $S$ to $S_\calU = \bigcup_{i\leq m} \set{ (z, y_i): z\in \calU(x_i)}$.{\tiny \tcp*{$S_\calU$ contains all possible perturbations of $S$.}}
  Set $m_0 = O(\vc(\calA)\vc(\calA)^*\log \vc(\calA)^*)$, and $T=O(\log|S_\calU|)$.\;
  \For{$1 \leq t\leq T$}{
    Set distribution $D_t$ on $S_\calU$ as in the $\alpha$-Boost algorithm.\;
    Sample $S'\sim D_t^{m_0}$, and project $S'$ to dataset $L\subseteq S$ by replacing each perturbation $z$ with its corresponding example $x$.\;
    Call \FMain on $L$, and denote by $f_t$ its output predictor.\;
  }
  Sample $N_{\rm co} = O\inparen{\vc^*(\calA)\log \vc^*(\calA)}$ i.i.d. indices $i_{1},\ldots,i_{N_{\rm co}} \sim {\rm Uniform}(\{1,\ldots,T\})$.
  \\~~~~(repeat previous step until $f=\MAJ(f_{i_1},\ldots,f_{i_{N_{\rm co}}})$ has $R_{\calU}(f;S)=0$)\;
\KwOutput{A majority-vote $\MAJ(f_{i_1},\dots, f_{i_{N_{\rm co}}})$ predictor.}
  \BlankLine
  \SetKwProg{Fn}{}{:}{\KwRet}
  \Fn{\FMain{Dataset $L$, Learner $\calA$}}{
        Inflate dataset $L$ to $L_\calU = \bigcup_{(x,y)\in L} \set{ (z, y): z\in \calU(x)}$, and set $T_L = O(\log|L_\calU|)$\;
        Run $\alpha$-Boost with black-box access to $\calA$ on $L_\calU$ for $T_L$ rounds.\;
        Let $h_1,\dots, h_{T_L}$ denote the hypotheses produced by $\alpha$-Boost with $T_L$ oracle queries to $\calA$.\;
        Sample $N = O\inparen{\vc^*(\calA)}$ i.i.d. indices $i_{1},\ldots,i_{N} \sim {\rm Uniform}(\{1,\ldots,T_L\})$. \\~~~~(repeat previous step until $f=\MAJ(h_{i_1},\ldots,h_{i_N})$ has $R_{\calU}(f;L)=0$)\;
        \KwRet $f=\MAJ(h_{i_1},\dots, h_{i_N})$\;}
  \BlankLine
  \SetKwProg{Fn}{}{:}{\KwRet}
  \Fn{\AlphaBoost{Dataset $L$, Learner $\calA$}}{
        Initialize $D_1$ to be uniform over $L$, and set $T_L=O(\log |L|)$.\;
        \For{$1 \leq t\leq T_L$}{
            Run $\calA$ on $S' \sim D_t^{m_0}$, and denote by $h_t$ its output. (repeat until ${\rm err}_{D_{t}}(h_t) \leq 1/3$)\;
            Compute a new distribution $D_{t+1}$ by applying the following update for each $(x,y)\in L$:
            \[ 
                D_{t+1}(x) = \frac{D_t(x)}{Z_t} \times \begin{cases} 
                                                              e^{-2\alpha}& \text{if }h_t(x)=y\\
                                                              1 &\text{otherwise}
                                                           \end{cases}
            \]
            where $Z_t$ is a normalization factor and $\alpha$ is set as in Lemma~\ref{lem:alphaboost}\;
        }
        \KwRet $h_1,\dots, h_{T_L}$.\;
  }
\end{algorithm}
\begin{theorem}
\label{thm:upperbound}
For any adversary $\calU$, Algorithm~\ref{alg:robustify} can {\em robustly PAC learn} any target class $\mathcal{C}$ using black-box oracle calls to any PAC learner $\calA$ for $\calC$ with:
\begin{enumerate}
    \item Sample Complexity $m=O\inparen{ \frac{d{d^*}^2\log^2d^*}{\eps}\log\inparen{\frac{d{d^*}^2\log^2d^*}{\eps}} + \frac{ \log(1/\delta)}{\eps}}$,
    \item Oracle Complexity $T=O\inparen{\inparen{\log{m}+\log|\calU|}^2 + \log(1/\delta)}$,
\end{enumerate}
where $d=\vc(\calA)$ and $d^*=\vc^*(\calA)$ are the primal and dual VC dimension of $\calA$.
\end{theorem}

Importantly, the sample complexity of Algorithm~\ref{alg:robustify} is independent of the number of allowed perturbations $\abs{\calU}$, in contrast to work by \citet{DBLP:conf/alt/AttiasKM19}, that can be interpreted as giving a reduction with sample complexity $m \propto \abs{\calU}\log\abs{\calU}$, and oracle complexity $T \propto \abs{\calU} \log^2 \abs{\calU}$.

Before proceeding with the proof of Theorem~\ref{thm:upperbound}, we briefly describe our strategy and its main ingredients. Given a dataset $S$ that is robustly realizable by some target concept $c\in\calC$, we show that we can use the non-robust learner $\calA$ to implement a $\RERM$ oracle that guarantees zero {\em empirical} robust loss on $S$ using \texttt{ZeroRobustLoss} in Algorithm~\ref{alg:robustify}. But what about the {\em population} robust loss? Our main goal to is adversarially robustly learn $\calC$ and not just minimize the empirical robust loss. Fortunately, we show that the arguments on \textit{robust} generalization based on sample compression in \cite[Theorem 4]{pmlr-v99-montasser19a} will still go through when we replace the $\RERM_\calC$ oracle they used with our \texttt{ZeroRobustLoss} procedure in Algorithm~\ref{alg:robustify}. This is achieved by showing that the image of \texttt{ZeroRobustLoss} has bounded VC dimension and \textit{dual} VC dimension. The following lemma, whose proof is provided in Appendix~\ref{sec:appendix}, bounds the \textit{dual} VC dimension of the convex-hull of a class $\calH$. This result might be of independent interest. 

\begin{lem}
\label{lem:dualvc-convexhull}
Let ${\rm co}^k(\calH)=\set{ x \mapsto \MAJ(h_1,\dots,h_k)(x): h_1,\dots,h_k \in \calH}$. Then, the dual VC dimension of $\co^k(\calH)$ satisfies $\vc^*(\co^k(\calH)) \leq O(d^* \log k)$.
\end{lem}

In addition, we state two extra key lemmas that will be useful for us in the proof. First, Lemma~\ref{lem:alphaboost} states that running $\alpha$-Boost on a dataset for enough rounds produces a sequence of predictors that achieve zero loss on the dataset (with a margin).
\begin{lem}[see, e.g., Corollary 6.4 and Section 6.4.3 in \cite{schapire:12}]
\label{lem:alphaboost}
Let $S=\set{(x_i, c(x_i))}_{i=1}^{m}$ be a dataset where $c \in \calC$ is some target concept, and $\calA$ an arbitrary PAC learner for $\calC$ (for $\eps = 1/3$, $\delta = 1/3$). Then, running $\alpha$-Boost (see description in Algorithm~\ref{alg:robustify}) on $S$ with black-box oracle access to $\calA$ with $\alpha = \frac{1}{2}\ln\inparen{1+\sqrt{\frac{2\ln m}{T}}}$ for $T=\lceil 112 \ln(m) \rceil = O(\log m )$ rounds suffices to produce a sequence of hypotheses $h_{1},\ldots,h_{T} \in {\rm im}(\calA)$ such that 
\[\forall (x,y) \in S, \frac{1}{T} \sum_{i=1}^{T} \ind[ h_{i}(x) = y ] \geq \frac{5}{9}.\]
In particular, this implies that the majority-vote $\MAJ(h_1,\dots,h_T)$ achieves zero error on $S$. 
\end{lem}

Second, Lemma~\ref{lem:sparse} describes a sparsification technique due to \cite{moran:16} which allows us to control the complexity of the majority-vote predictors that we use in Algorithm~\ref{alg:robustify}.

\begin{lem} [Sparsification of Majority Votes, \cite{moran:16}]
\label{lem:sparse}
Let $\calH$ be a hypothesis class with finite primal and dual VC dimension, and $h_1,\dots, h_T$ be predictors in $\calH$. Then, for any $(\eps,\delta) \in (0,1)$, with probability at least $1-\delta$ over $N=O\inparen{\frac{\vc^*(\calH) + \log(1/\delta)}{\eps^2}}$ independent random indices $i_{1},\ldots,i_{N} \sim {\rm Uniform}(\{1,\ldots,T\})$,
we have:
\[
\forall (x,y)\in \calX\times \calY, \abs{ \frac{1}{N} \sum\limits_{j=1}^{N} \ind[ h_{i_{j}}(x) = y ] -  \frac{1}{T} \sum\limits_{i=1}^{T} \ind[ h_{i}(x) = y ]} < \eps.
\]
\end{lem}

We are now ready to proceed with the proof of Theorem~\ref{thm:upperbound}.
\begin{proof}[Proof of Theorem~\ref{thm:upperbound}]
Let $\calU$ be an arbitrary adversary. Let $\calC$ be a target class that is PAC learnable with some PAC learner $\calA$. Let $\calH$ denote the base class of hypotheses of learner $\calA$. Let $d$ denote the VC dimension of $\calH$, and $d^*$ denote the dual VC dimension of $\calH$. Our proof is divided into two parts.

\textbf{Zero Empirical Robust Loss.}~Let $L=\set{(x_1,y_1),\dots,(x_m,y_m)}$ be a dataset that is {\em robustly} realizable with some target concept $c \in \calC$; in other words, for each $(x,y)\in L$ and each $z\in \calU(x)$, $c(z)=y$. We  will show that we can use the non-robust learner $\calA$ to guarantee zero {\em empirical} robust loss on $L$. This procedure is described in \texttt{ZeroRobustLoss} in Algorithm~\ref{alg:robustify}. We inflate dataset $L$ to include all possible perturbations under the adversary $\calU$. Let $L_\calU = \bigcup_{i\leq m} \set{ (z, y_i): z\in \calU(x_i)}$ denote the inflated dataset. Observe that $|L_\calU|\leq m |\calU|$, since each point $x\in \calX$ has at most $|\calU|$ possible perturbations. We run the $\alpha$-Boost algorithm on the inflated dataset $L_\calU$ with {\em black-box} access to PAC learner $\calA$, where in each round of boosting $m_0$ samples are fed to $\calA$ (where $m_0$ is chosen Step 2). By Lemma~\ref{lem:alphaboost}, running $\alpha$-Boost with $T = O\inparen{\log(|L_\calU|)}$ oracle calls to $\calA$ suffices to produce 
a sequence of hypotheses $h_{1},\ldots,h_{T} \in \calH$ such that 
\[\forall (z,y) \in L_{\calU}, \frac{1}{T} \sum_{i=1}^{T} \ind[ h_{i}(z) = y ] \geq \frac{5}{9}.\]

Specifically, the above implies that a majority-vote over hypotheses $h_{1},\ldots,h_{T}$ achieves zero {\em robust} loss on dataset $L$, $\Risk_{\calU}(\MAJ(h_1,\dots,h_T); L)=0$. By Step 13 in \texttt{ZeroRobustLoss} in Algorithm~\ref{alg:robustify} and Lemma~\ref{lem:sparse} (with $\eps=1/18, \delta=1/3$), we have that for $N = O(d^*)$, the sampled predictors $h_{i_1},\dots, h_{i_{N}}$ satisfy
\[
\forall (z,y)\in L_\calU, \frac{1}{N} \sum\limits_{j=1}^{N} \ind[ h_{i_{j}}(z) = y ] > \frac{1}{T} \sum\limits_{i=1}^{T} \ind[ h_{i}(z) = y ]  - \frac{1}{18} > \frac{5}{9} - \frac{1}{18} = \frac{1}{2}.
\]
Therefore, the majority-vote over the sampled hypotheses $\MAJ(h_{i_1}, \dots, h_{i_N})$ achieves zero robust loss on $L$, $\Risk_{\calU}(\MAJ(h_{i_1}, \dots, h_{i_N}); S)=0$. Thus, we can implement a $\RERM$ oracle (see Equation~\ref{eqn:rerm}) using the procedure \texttt{ZeroRobustLoss} in Algorithm~\ref{alg:robustify}. The sparsification step (Step 12) controls the complexity 
of the image of \texttt{ZeroRobustLoss}, i.e., the hypothesis class that is being implicitly used. Specifically, observe that the sparsified predictor $f=\MAJ(h_{i_1}, \dots, h_{i_N})$ lives in  ${\rm co}^{O(d^*)}(\calH)$, which is the convex-hull of $\calH$ that combines at most $O(d^*)$ predictors. To guarantee \textit{robust} generalization in the next part, it suffices to bound the VC dimension and dual VC dimension of ${\rm co}^{O(d^*)}(\calH)$. By \citep{blumer:89}, the VC dimension of ${\rm co}^{O(d^*)}(\calH)$ is at most $O(dd^*\log d^*)$, and by Lemma~\ref{lem:dualvc-convexhull}, the dual VC dimension of ${\rm co}^{O(d^*)}(\calH)$ is at most $O(d^*\log d^*)$. 

\textbf{Robust Generalization through Sample Compression.}~This part builds on the approach of \citet[Theorem 4]{pmlr-v99-montasser19a}. Specifically, we observe that their proof works even if we replace the $\RERM_\calC$ oracle they used, with our \texttt{ZeroRobustLoss} procedure in Algorithm~\ref{alg:robustify} that is described above. We provide a self-contained analysis below.

Let $\calD$ be an arbitrary distribution over $\calX\times \calY$ that is robustly realizable with some concept $c\in \calC$, i.e., $\Risk_\calU(c;\calD)=0$. Fix $\eps, \delta \in (0,1)$ and a sample size $m$ that will be determined later. Let $S=\set{(x_1,y_1),\dots,(x_m,y_m)}$ be an i.i.d. sample from $\calD$. We run the $\alpha$-Boost algorithm (see Algorithm~\ref{alg:robustify}) on the inflated dataset $S_\calU$, this time with \texttt{ZeroRobustLoss} as the subprocedure. Specifically, on each round of boosting, $\alpha$-Boost computes an empirical distribution $D_t$ over $S_\calU$ (according to Step 18). We draw $m_0 = O(dd^*\log d^*)$ samples $S'$ from $D_t$, and \textit{project} $S'$ to a dataset $L_t\subset S$ by replacing each perturbation $(z,y)\in S'$ with its corresponding original point $(x,y) \in S$, and then we run \texttt{ZeroRobustLoss} on dataset $L_t$. The projection step is crucial for the proof to work, since we use a \emph{sample compression} argument to argue about \textit{robust} generalization, and the sample compression must be done on the \textit{original} points that appeared in $S$ rather than the perturbations in $S_\calU$. 

By classic PAC learning guarantees \citep{vapnik:74,blumer:89}, with $m_0 = O(\vc({\rm co}^{O(d^*)}(\calH))) = O(dd^*\log d^*)$, we are guaranteed uniform convergence of $0$-$1$ risk over predictors in ${\rm co}^{O(d^*)}(\calH)$ (the effective hypothesis class used by \texttt{ZeroRobustLoss}). So, for any distribution $D$ over $\calX \times \calY$ with $\inf_{c \in \calC} \err(c;\calD)=0$, with nonzero probability over $S'\sim \calD^{m_0}$, every $f\in {\rm co}^{O(d^*)}(\calH)$ satisfying $\err_{S'}(f)=0$, also has $\err_D(f)<1/3$. 
By the guarantee of \texttt{ZeroRobustLoss} (established above), we know that $f_t=\texttt{ZeroRobustLoss}(L_t,\calA)$ achieves zero robust loss on $L_t$, $\Risk_\calU(f_t; L_t)=0$, which by definition of the projection means that $\err_{S'}(f_t)=0$, and thus $\err_{D_t}(f_t)<1/3$. This allows us to use \texttt{ZeroRobustLoss} with $\alpha$-Boost to establish a \textit{robust} generalization guarantee. Specifically, Lemma~\ref{lem:alphaboost} implies that running the $\alpha$-Boost algorithm  
with $S_{\calU}$ as its dataset for $T = O(\log(|S_{\calU}|))$ rounds, 
using \texttt{ZeroRobustLoss} to produce the hypotheses $f_t \in {\rm co}^{O(d^*)}(\calH)$ for the distributions $D_{t}$ produced on each round of the algorithm, will produce a sequence of hypotheses $f_1,\dots,f_T \in {\rm co}^{O(d^*)}(\calH)$ such that:

\[\forall (z,y) \in S_{\calU}, \frac{1}{T} \sum_{i=1}^{T} \ind[ f_{i}(z) = y ] \geq \frac{5}{9}.\]

Specifically, this implies that the majority-vote over hypotheses $f_{1},\ldots,f_{T}$ achieves zero {\em robust} loss on dataset $S$, $\Risk_{\calU}(\MAJ(f_1,\dots,f_T); S)=0$. Note that each of these classifiers $f_{t}$ is equal to $\texttt{ZeroRobustLoss}(L_{t}, \calA)$ for some $L_{t} \subseteq S$ with $|L_{t}|=m_0$. Thus, the classifier $\MAJ(f_1,\dots,f_T)$ is representable as the value of an (order-dependent) reconstruction function $\phi$ with 
a compression set size
\[
m_0T = O(\vc({\rm co}^{O(d^*)}(\calH)) \log(|S_{\calU}|)) = O(dd^*\log d^* \inparen{\log m + \log |\calU|}).\]

This is not enough, however, to obtain a sample complexity bound that is independent of $|\calU|$. For that, we will sparsify the majority-vote as in Step 7 in Algorithm~\ref{alg:robustify}. Lemma~\ref{lem:sparse} (with $\eps=1/18,\delta=1/3$) guarantees that for $N_{\rm co} = O(d^*\log d^*)$, the sampled predictors $f_{i_1},\dots, f_{i_{N_{\rm co}}}$ satisfy:

\[\forall (z,y)\in S_\calU, \frac{1}{N_{\rm co}} \sum\limits_{j=1}^{N_{\rm co}} \ind[ f_{i_{j}}(z) = y ] > \frac{1}{T} \sum\limits_{i=1}^{T} \ind[ f_{i}(z) = y ]  - \frac{1}{18} > \frac{5}{9} - \frac{1}{18} = \frac{1}{2},\]

so that the majority-vote achieves zero robust loss on $S$, $\Risk_\calU(\MAJ(f_{i_1},\dots, f_{i_{N_{\rm co}}}); S)=0$. Since again, each $f_{i_j}$ is the result of $\texttt{ZeroRobustLoss}(L_{t}, \calA)$ for some $L_{t} \subseteq S$ with $|L_{t}|=m_0$, we have that the classifier $\MAJ(f_{i_1},\dots, f_{i_{N_{\rm co}}})$ can be represented as the value of an (order-dependent) reconstruction function $\phi$ with a compression set size $m_0 N_{\rm co}=O(dd^*\log d^* \cdot d^*\log d^*)=O(d{d^*}^2\log^2(d^*))$. Lemma~\ref{lem:robust-compression} (\cite{pmlr-v99-montasser19a}) which extends to the robust loss the classic compression-based generalization guarantees from the $0$-$1$ loss, implies that for $m \geq c d{d^*}^2\log^2(d^*)$ (for an appropriately large numerical constant $c$), with probability at least $1-\delta$ over $S\sim \calD^m$, 
$$R_{\calU}(\MAJ(f_{i_1}, \dots, f_{i_{N_{\rm co}}});\calD) \leq O\!\left( \frac{d{d^*}^2\log^2(d^*)}{m} \log(m) + \frac{1}{m} \log(1/\delta) \right).$$
Setting this less than $\eps$ and solving for a sufficient size of $m$ to achieve this yields the stated sample complexity bound. 

Our oracle complexity $T$ (number of calls to $\calA$) is at most $O((\log|S_\calU|)^2 + \log(1/\delta))\leq O\inparen{\inparen{\log{m}+\log|\calU|}^2 + \log(1/\delta)}$, since \texttt{ZeroRobustLoss} in Algorithm~\ref{alg:robustify} terminates in at most $O(\log|S_\calU|)$ rounds each time it is invoked, and it is invoked at most $O(\log|S_\calU|)$ times by the outer $\alpha$-Boost algorithm in Algorithm~\ref{alg:robustify}.  Therefore, we have at most $O\inparen{\inparen{\log{m}+\log|\calU|}^2}$ geometric random variables that represent that number of times Step 19 is invoked, which is the step where learner $\calA$ is called. The success probability of Step 19 is a constant (say $2/3$), therefore the mean of the sum of the geometric random variables is $O\inparen{\inparen{\log{m}+\log|\calU|}^2}$. Since sums of geometric random variables concentrate around the mean \citep{brown2011wasted}, we get that with probability at least $1-\delta$, the total number of times Step 19 is executed is at most $O\inparen{\inparen{\log{m}+\log|\calU|}^2 + \log(1/\delta)}$. This concludes the proof.
\end{proof}

\subsection{A Lowerbound on the Oracle Complexity}
The oracle complexity of Algorithm~\ref{alg:robustify} depends on $\log|\calU|$.  Can this dependence  be reduced or avoided?   Unfortunately, we show in \autoref{thm:lowerbound} that the dependence on $\log|\calU|$ is unavoidable and {\em no} reduction algorithm can do better.%

It is relatively easy to show, by relying on lower bounds from the boosting literature \citep[Section 13.2.2]{schapire:12}, that for any reduction algorithm, there exists a target class $\calC$ with $\vc(\calC)\leq 1$ and a PAC learner $\calA$ for $\calC$ such that $\Omega\inparen{\log|\calU|}$ oracle calls to this PAC learner are necessary to achieve zero {\em empirical} robust loss.  But this is done by constructing a ``crazy'' improper learner with $\vc(\calA)\propto |\calU| \gg \vc(\calC)$.

Perhaps we can ensure better oracle complexity by requiring a more constrained PAC learner $\calA$, e.g.~with low VC dimension (as in our upper bound), or perhaps even a proper learner, or an $\ERM$, or maybe where $\im(\calA)$ (and so also $\calC$) is finite. We next present a lower bound to show that none of these help improve the oracle complexity. Specifically, we will present a construction showing that for any reduction algorithm $\calB$ there is a randomized target class $\calC$ and a PAC learner $\calA$ for $\calC$ with $\vc(\calA)=1$ where $\calB$ needs to make $\Omega(\log|\calU|)$ oracle calls to $\calA$ to robustly learn $\calC$. The idea here is that the target class $\calC$ is chosen randomly after $\calB$, and so $\calB$ essentially knows nothing about $\calC$ and needs to communicate with $\calA$ in order to learn. As a reminder, a reduction algorithm has a budget of $T$ oracle calls to a non-robust learner $\calA$, where each oracle call is constructed with $m_0$ points, 
or more generally a \emph{distribution} over $\calX \times \calY$. We next show that 
any successful reduction requires $T=\Omega(\log|\calU|)$ for some non-robust learner $\calA$.

\begin{theorem}
\label{thm:lowerbound}
For any sufficiently large 
integer $u$, if $|\calX| \geq u^{10 u}$,
there exists an adversary $\calU$ 
with $|\calU|=u$
such that for any reduction algorithm $\mathcal{B}$ and for any $\eps > 0$, there exists a target class $\mathcal{C}$ and a PAC learner $\mathcal{A}$ for $\calC$ with $\vc(\calA)=1$ such that, if the training sample has size at most $(1/8)|\calU|^9$, then $\calB$ needs to make $T\geq  \frac{\log \inabs{\calU}}{\log(2/\eps)}$ oracle calls to $\calA$ in order to robustly learn $\calC$.
\end{theorem}

\begin{proof}
We begin with describing the construction of the adversary $\calU$. Let $m\in \bbN$; we will 
construct $\calU$ with $|\calU|=2^m$, 
supposing $|\calX| \geq 2\binom{2^{10 m}}{2^m}+2^{10 m}$. 
Let $Z = \set{z_1, \dots, z_{2^{10m}}}\subset \calX$ be a set of $2^{10m}$ unique points from $\calX$. For each subset $L\subset Z$ where $\inabs{L}=2^m$, pick a unique pair $x^+_{L}, x^-_{L}\in \calX \setminus Z$ and define $\calU(x^+_{L})=\calU(x^-_{L})=L$. That is, for every choice $L$ of $2^{m}$ perturbations from $Z$, there is a corresponding pair $x^+_L, x_L^-$ where $\calU(x^+_L)= \calU(x^-_L)=L$. For any point $x \in \calX \setminus Z$ that is remaining, define $\calU(x)=\set{}$.

Let $\calB$ be an arbitrary reduction algorithm, and let $\eps>0$ be the error requirement. We will now describe the construction of the target class $\calC$. The target class $\calC$ will be constructed randomly. Namely, we will first define a labeling $\Tilde{h}: Z \to \calY$ on the perturbations in $Z$ that is positive on the first half of $Z$ and negative on the second half of $Z$: $\Tilde{h}(z_i) = +1$ if $i\leq \frac{2^{10m}}{2}$, and $\Tilde{h}(z_i)=-1$ if $i> \frac{2^{10m}}{2}$. Divide the positive/negative halves into groups of size $2^m$:
\[ \underbrace{\set{\text{first $2^m$ positives}}}_{G^{+}_{1}}, \dots, \underbrace{\set{\text{last $2^m$ positives}}}_{G^{+}_{2^{9m-1}}} \Big\lvert  \underbrace{\set{\text{first $2^m$ negatives}}}_{G^{-}_{1}}, \dots, \underbrace{\set{\text{last $2^m$ negatives}}}_{G^{-}_{2^{9m-1}}}.\]

Let $\eps'=\eps/2$. The target concept $h^*: \calX \to \calY$ 
is generated by randomly flipping the labels of an $\eps'$ fraction of the points in each group $G^+_1, \dots, G^+_{2^{9m-1}}$ from positive to negative and randomly flipping the labels of an $\eps'$ fraction of the points in each group $G^-_1,\dots,G^-_{2^{9m-1}}$ from negative to positive. This defines $h^*$ on $Z$; then for every pair $x^+,x^-\in \calX \setminus Z$ where $\calU(x^+)=\calU(x^-) \neq \{\}$, define $h^*(x^+) = +1$ and $h^*(x^-) = -1$. Once $h^*$ is generated, we define the distribution $D_{h^*}$ over $\calX \times \calY$ that will be used in the lower bound by swapping the $\eps'$ fractions of points with the flipped labels in each pair $(G^+_1, G^-_1), \dots, (G^+_{2^{9m-1}}, G^-_{2^{9m-1}})$ which defines new positive/negative pairs: $(G(h^*)^+_1, G(h^*)^-_1), \dots, (G(h^*)^+_{2^{9m-1}}, G(h^*)^-_{2^{9m-1}})$. Let $x^+_i, \_ = \calU^{-1}(G(h^*)^+_i)$ and $\_, x^-_i = \calU^{-1}(G(h^*)^-_i)$ for each $i\in [2^{9m-1}]$ ($\calU^{-1}$ returns a pair of points). Observe that by definition of $h^*$ on $\calX \setminus Z$, we have that $h^*(x_i^+) = +1$ and $h^*(x_i^-) = -1$ since $h^*(z) = +1 \forall z\in G(h^*)^+_i$ and $h^*(z) = -1 \forall z\in G(h^*)^-_i$. Let ${D_{h^*}}$ be a uniform distribution over $(x^+_1, +1), (x^-_1, -1),\dots, (x^{+}_{2^{9m-1}}, +1), (x^{-}_{2^{9m-1}}, -1)$. 

Let $T \leq \frac{\log{2^m}}{\log(1/\eps')}$. Define a randomly-constructed target class $\calC = \set{h_1, \dots, h_T, h_{T+1}}$ where $h_{T+1} = h^*$ and $h_1, h_2, \dots, h_T$ are generated according the following process: If $t=1$, then $h_1 := \Tilde{h}$ (augmented to 
all of $\calX$ by letting $\Tilde{h}(x) = h^*(x)$ 
for all $x \in \calX \setminus Z$). For $t\geq 2$, let ${\rm DIS}_{t-1}=\set{ z\in Z: h_{t-1}(z) \neq h^*(z)}$, and construct $h_t$ by flipping a uniform randomly-selected $1-\eps'$ fraction of the labels of $h_{t-1}$ in $G^+_i \cap {\rm DIS}_{t-1}$ and $1-\eps'$ fraction of the labels of $h_{t-1}$ in $G^-_i \cap {\rm DIS}_{t-1}$ for each $i\in [2^{9m-1}]$. Observe that by construction, $h_1,\dots, h_T$ satisfy the property that they agree with $h^*$ on $\calX \setminus Z$, i.e.\, $h_t(x)=h^*(x)$ for each $t\leq T$ and each $x\in \calX \setminus Z$. 

We now state a few properties of the randomly-constructed target class $\calC$ that we will use in the remainder of the proof. First, observe that by definition of ${\rm DIS}_t$ for $t\leq T$, we have that $G^{\pm}_i \cap {\rm DIS}_{T} \subseteq G^{\pm}_i \cap {\rm DIS}_{T-1} \subseteq \dots \subseteq G^{\pm}_i \cap {\rm DIS}_{1}$ for each $1 \leq i \leq 2^{9m-1}$. In addition,  
\[|G^{\pm}_i \cap {\rm DIS}_{t}| \geq \eps'|G^{\pm}_i \cap {\rm DIS}_{t-1}| ~~\text{for each } 1 \leq i \leq 2^{9m-1}.\]
By the random process generating $h^*$, we also know that $|G^{\pm}_i \cap {\rm DIS}_1| \geq \eps' 2^m$.
Combined with the above, this implies that:
\[ |G^{\pm}_i \cap {\rm DIS}_{T}| \geq {\eps'}^T 2^m ~~\text{for each } 1 \leq i \leq 2^{9m-1}.\]

So, for $T \leq \frac{\log{2^m}}{\log(2/\eps)}$, we are guaranteed that $|G^{\pm}_i \cap {\rm DIS}_{T}|\geq 1$ for each $1 \leq i \leq 2^{9m-1}$.

We now describe the construction of a PAC learner $\calA$ with $\vc(\calA) = 1$ for the randomly generated concept $h^*$ above; we assume that $\calA$ knows $\calC$ (but of course, $\calB$ does not know $\calC$). 
\begin{algorithm}[H]
\caption{Non-Robust PAC Learner $\calA$}\label{alg:non-robust}
\SetKwInput{KwInput}{Input}                % Set the Input
\SetKwInput{KwOutput}{Output}              % set the Output
\SetKwFunction{FMain}{ZeroRobustLoss}
\DontPrintSemicolon
  \KwInput{Distribution $P$ over $\calX$.}
\KwOutput{$h_s$ for the 
\emph{smallest} $s \in [T]$ with $\err_{P}(h_s,h^*) \leq \eps$ (or outputting $h_{T+1} = h^*$ if 
no such $s$ exists).}
\end{algorithm}

First, we will show that $\vc(\calA) = 1$. By definition of $\calA$, it suffices to show that $\vc{(\calC)}=\vc(\set{h^*, h_1,\dots,h_T})=1$. By definition of $h^*$ and $h_1$, it is easy to see that there is a $z\in Z$ where $h^*(z) \neq h_1(z)$, and thus $\vc(\calC)\geq 1$. Observe that by construction, each predictor in $h_1, \dots, h_T$ operates as a threshold in each group $G^{+}_1,G^{-}_1, \dots, G^{+}_{2^{9m-1}}, G^{-}_{2^{9m-1}}$
(ordered according to the order in which 
the labels are flipped in the $h_1,\ldots,h_T$ 
sequence). As a result, each $x\in \calX$ has its label flipped at most once in the sequence $\inparen{h_1(x), \dots, h_T(x), h^*(x)}$. This is because once the ground-truth label of $x$, $h^*(x)$, is revealed by some $h_t$ (i.e., $h_t(x)=h^*(x)$), all subsequent predictors $h_{t'}$ satisfy $h_{t'}(x) = h^*(x)$. Thus, for any two points $z,z' \in \calX$, the number of possible behaviors $\inabs{\set{ (h(z), h(z')): h\in \calC }} \leq 3$. Therefore, $\calC$ cannot shatter two points. This proves that $\vc(\calC)\leq 1$.

\paragraph{Analysis} Suppose that we run the reduction algorithm $\calB$ with non-robust learner $\calA$ for $T$ rounds to obtain predictors $h_{s_1}=\calA(P_1), \dots, h_{s_T} = \calA(P_T)$. We will show that $\Prob_{h^*}\insquare{ s_T \leq T | S} > 0$, meaning that with non-zero probability learner $\calA$ will not reveal the ground-truth hypothesis $h^*$. For $t\leq T$, let $E_t$ denote the event that $\err_{P_t}(h_{s_{t-1}+1}, h^*) \leq \eps$. When conditioning on $S,s_1,\dots, s_{t-1}$, observe that by construction of the randomized hypothesis class $\calC$, for each $i\leq 2^{9m-1}$ such that $\{(x^-_i,-1),(x^+_i,+1)\} \cap S = \emptyset$, and each $z \in G^{\pm}_i \cap {\rm DIS}_{s_{t-1}}: \Prob_{h^*}\insquare{h^*(z) \neq h_{s_{t-1} + 1}(z) | S, s_1, \dots, s_{t-1}} \leq \eps'=\eps/2$. It follows then by the law of total probability that for any distribution $P_t$ constructed by $\calA$:
\[\Ex_{h^*} \insquare{ \err_{P_t}(h_{s_{t-1}+1}, h^*) | S, s_1, \dots, s_{t-1}} \leq \frac{\eps}{2}.\]

By Markov's inequality, it follows that
\begin{align*}
    \Prob_{h^*}\insquare{ \Bar{E_t} | S,s_1,\dots,s_{t-1}} &= \Prob_{h^*}\insquare{ \err_{P_t}(h_{s_{t-1}+1}, h^*) > \eps | S,s_1,\dots,s_{t-1}}\\
    &\leq \frac{\Ex_{h^*} \insquare{ \err_{P_t}(h_{s_{t-1}+1}, h^*) | S, s_1, \dots, s_{t-1}}}{\eps} \leq \frac{1}{2}.
\end{align*}

By law of total probability,
\begin{align*}
    \Prob_{h^*}\insquare{s_T \leq T |S} \geq \Prob_{h^*}\insquare{E_1 | S} \times \Prob_{h^*} \insquare{ E_2 | S, E_1} \times \cdots \times \Prob_{h^*} \insquare{ E_T | S, E_1, \dots, E_{T-1}} \geq \inparen{\frac{1}{2}}^T > 0.
\end{align*}

To conclude the proof, we will show that if the reduction algorithm $\calB$ sees at most $1/2$ of the support of distribution $D_{h^*}$ through a training set $S$ and makes only $T \leq \frac{\log 2^m}{\log(2/\eps)}$ oracle calls to $\calA$, then it will likely fail in robustly learning $h^*$. For each 
$i \leq 2^{9m-1}$, 
conditioned on the event that 
$\{(x^-_i,-1),(x^+_i,+1)\} \cap S = \emptyset$,
and conditioned on $h_{s_1},\ldots,h_{s_T}$, 
there is a $z \in Z$ that is equally likely 
to be in $\calU(x^-_i)$ or $\calU(x^+_i)$. To see why such a point exists, we first describe an equivalent distribution generating $h^*,h_1,\dots,h_T$. For each $i\leq 2^{9m-1}$ randomly select a $2\eps'$ fraction of points from $G_i^{+}$ and a $2\eps'$ fraction of points from $G_i^{-}$. Then, randomly pair the points in each $2\eps'$ fraction to get $\eps'2^m$ pairs $z_i,z'_i$ for each $G_i^{\pm}$. For each pair $z_i,z_i'$ flip a fair coin $c_i$: if 
$c_i = 1$, $z_i$'s label gets flipped 
and otherwise if $c_i = 0$ then 
$z'_i$'s label gets flipped. This is equivalent to generating $h^*$ by flipping the labels of a uniform randomly-selected $\eps$ fraction of points in each $G_i^{\pm}$ as originally described, but is helpful book-keeping that simplifies our analysis. In addition, $h_1,\dots,h_T$ can be generated in a similar fashion. Since $T \leq \frac{\log{2^m}}{\log(2/\eps)}$, we are guaranteed that $|G^{\pm}_i \cap {\rm DIS}_{s_T}|\geq 1$. By definition of ${\rm DIS}_{s_T}$, this implies that that there is a pair of points $z_i,z'_i$ in each $G_i^{\pm}$ where each $h_{s_t}(z_i)=h_{s_t}(z'_i)$ for $t\leq T$ but $h^*(z_i)\neq h^*(z'_i)$
(i.e., each $h_{s_t}$ never reveals the ground-truth label for at least one pair). And then in the 
end, if $\{(x^-_i,-1),(x^+_i,+1)\} \cap S = \emptyset$, 
$\calB$ will make some prediction on $z_i$, 
and the posterior probability of it being 
wrong is $1/2$.
More formally, for any training dataset $S\sim D_{h^*}^{|S|}$ where $|S|\leq 2^{9m-3}$, any $h_{s_1},\dots, h_{s_T}$ returned by $\calA$ where $T \leq \frac{\log 2^m}{\log(2/\eps)}$, and any predictor $f:\calX \to \calY$ that is picked by $\calB$: 
\begin{align*}
    &\Ex_{h^*} \insquare{\Risk_{\calU}(f; D_{h^*}) \lvert S,h_{s_1},\ldots,h_{s_T}} \geq \Ex_{h^*} \insquare{ \frac{1}{2^{9m}} \sum_{\substack{(x,y) \notin S,
    \\(x,y) \in {\rm supp}(D_{h^*})}} \sup_{z \in \calU(x)} \ind[f(z) \neq y]  \Bigg\lvert S,h_{s_1},\ldots,h_{s_T}}\\
    &= \frac{1}{2^{9m}} 
    \sum_{i = 1}^{2^{9m-1}} \Prob_{h^*} \left[  \inparen{(x^+_i,+1), (x^-_i, -1) \notin S} ~\land~ \right.
    \\ & {\hskip 24mm}\left. \inparen{\exists z \in \calU(x^+_i) \text{ s.t. } f(z)\neq +1  \vee \exists z \in \calU(x^-_i) \text{ s.t. } f(z) \neq -1}  \Bigg\lvert S,h_{s_1},\ldots,h_{s_T}\right]\\
    &\geq \frac{2^{9m-1}}{2^{9m}} \frac{1}{2} = \frac{1}{4}. 
\end{align*}

This implies that, for any $\calB$
limited to $n \leq 2^{9m-3}$ 
training examples and $T \leq \frac{m}{\log_{2}(2/\eps)}$ queries, 
there exists a \emph{deterministic} choice of 
$h^*$ and $h_1,\ldots,h_T$, and a
corresponding learner $\calA$ that is a PAC 
learner for $\{h^*\}$ using hypothesis class 
$\{h^*,h_1,\ldots,h_T\}$ of VC dimension $1$, 
such that,  
for $S \sim D_{h^*}^{n}$, 
$\Ex_{S}[ \Risk_{\calU}(f;D_{h^*})] \geq \frac{1}{4}$. 
\end{proof}

\textbf{Computational Efficiency} Although the sample complexity of Algorithm~\ref{alg:robustify} is independent of $|\calU|$, we showed that the $\log|\calU|$ dependence in oracle complexity is unavoidable. This implies that the runtime of Algorithm~\ref{alg:robustify} will be at best weakly polynomial and  have at least a $\log|\calU|$ dependence. But maybe this is not so bad, because it is equivalent to the number of bits required to represent the adversarial perturbations. This weak poly-time dependence is common in almost all optimization algorithms (gradient descent, interior point methods, etc). What is more concerning is the linear runtime and memory dependence on $|\calU|$ that emerges from the explicit representation of the adversarial perturbations during training. In practice, many of the adversarial perturbations $\calU$ are infinite, but specified implicitly, and not by enumerating over all possible perturbations (e.g.~$\ell_p$ perturbations). This motivates the following next steps: What operations do we need to be able to implement efficiently on $\calU$ in order to robustly learn?  What access (oracle calls, or ``interface'') do we need to $\calU$?

\textbf{Sampling Oracle Over Perturbations} A reasonable form of access to $\calU$ that is sufficient for implementing Algorithm \ref{alg:robustify} is a sampling oracle that takes as input a point $x$ and an energy function $E: \calX \to \bbR$, and does the following:
\begin{enumerate}[label=(\alph*)]
    \item Samples a perturbation $z$ from a distribution given by $p_x(z) \propto \exp(E(z))*\ind[z\in \calU(x)]$. That is, the oracle samples from the set $\calU(x)$ based on the weighting encoded in $E$. 
    \item Calculates $\Prob\insquare{z\in \calU(x)}$ for the distribution given by $p(z) \propto \exp(E(z))$. 
\end{enumerate}
With such an oracle, Algorithm~\ref{alg:robustify} can be implemented without the need to do explicit inflation of $S$ to $S_\calU$, and can avoid the linear dependence on $|\calU|$. This is because Algorithm~\ref{alg:robustify} and its subprocedure \texttt{ZeroRobustLoss} just need to sample from distributions over the inflated set $S_\calU$ that are constructed by $\alpha$-Boost (as required in Steps 5 and 17 in Algorithm~\ref{alg:robustify}). This can be simulated via a two-stage process where we maintain a conditional distribution over $S$ (the original points), and then draw perturbations using the sampling oracle. Specifically, to sample from a distribution $D_t$ that is constructed by $\alpha$-Boost, we use two energy functions $E^{+}_t(z) = -2\alpha \sum_{i \leq t} \ind[g_t(z) = +1]$ and $E^{-}_t(z) = -2\alpha \sum_{i \leq t} \ind[g_t(z) = -1]$, where $g_1,\dots,g_t$ represent the sequence of predictors produced during the first $t$ rounds of boosting (either $h_t$'s produced by non-robust learner $\calA$ or $f_t$'s produced by \texttt{ZeroRobustLoss}). Using the sampling oracle, we can sample from $D_t$, by first sampling $(x,y)$ from $S$ based on the marginal estimates computed by the oracle (operation (b) described above) using energy function $E^{y}_t$, and then sampling $z$’s from their $\calU(x)$ (operation (a) described above) using energy function $E^{y}_t$.

\section{Learning with a Mistake Oracle for Adversarial Perturbations}
\label{sec:mistake-oracle}

In Section~\ref{sec:explicit}, we considered an explicit form of access to the set of adversarial perturbations $\calU$, as well as access via a sampling oracle. A more realistic form of access is having a mistake oracle for $\calU$:\begin{definition} [Mistake Oracle]
\label{def:mistake-oracle}
Denote by $\mathsf{O}_{\calU}$ a mistake oracle for $\calU$.
$\mathsf{O}_{\calU}$ takes as input a predictor $f:\calX\to \calY$ and an example $(x,y)$ and either: (a) asserts that $f$ is robust on $(x,y)$ (i.e.\ $\forall z\in \calU(x), f(z)=y$), or (b) returns an adversarial perturbation $z\in\calU(x)$ such that $f(z)\neq y$.
\end{definition}

Having only a mistake oracle for $\calU$, rather than an explicit representation, is a more realistic form of access. In this case, the reduction algorithm has no explicit knowledge of the set of adversarial perturbations $\calU$ and is forced to interact with the mistake oracle $\mathsf{O}_{\calU}$ in order to learn an adversarially robust predictor. Furthermore, what is typically referred to as adversarial training in practice fits exactly into this framework \citep{DBLP:conf/iclr/MadryMSTV18}. 

Does access to a mistake oracle $\mathsf{O}_{\calU}$ suffice to robustly learn a target class $\calC$ using a black-box non-robust learner $\calA$ for $\calC$? First, can we achieve a similar upper bound as in Theorem~\ref{thm:upperbound} but with a mistake oracle $\mathsf{O}_{\calU}$ rather than explicit access to $\calU$? Unfortunately, even with a black-box $\ERM$ for $\calC$, one can show that $\abs{\calU}$ oracle calls to $\mathsf{O}_{\calU}$ are unavoidable (proof provided in Appendix~\ref{sec:appendix}):
\begin{claim}
\label{claim:lowerbound-mistakeoracle}
For any reduction algorithm $\calB$, there exists an adversary $\calU$, target class $\calC$, and an $\ERM$ for $\calC$ with VC dimension 1, such that $\calB$ needs to make $T\geq \abs{\calU}$ oracle calls to $\mathsf{O}_{\calU}$.
\end{claim}

Thus, a non-robust PAC learner $\calA$ for $\calC$ is not enough to learn $\calC$ robustly with a mistake oracle $\mathsf{O}_{\calU}$ \emph{without a linear dependence on $\abs{\calU}$}. This suggests that a stronger assumption about $\calA$ is required. We next show that an {\em online} learner $\calA$ for $\calC$ suffices to robustly learn $\calC$ with a mistake oracle $\mathsf{O}_{\calU}$ {\em and} without any dependence on $\abs{\calU}$. Before proceeding, we include a brief reminder of what it means to learn in an online setting with a finite mistake bound:

\begin{definition} (Mistake Bound Model) 
We say an online learner $\calA$ learns a hypothesis class $\calC$ with mistake bound $M_\calA$ if learner $\calA$ makes at most $M_\calA$ mistakes on any sequence of examples that are labeled with some concept $c\in \calC$. 
\end{definition}

We are now ready to state our main result for this section. 

\begin{theorem}
\label{thm:online}
Algorithm~\ref{alg:robust-learner-mistakeO} robustly PAC learns any target class $\calC$ w.r.t.\ an adversary $\calU$ with black-box access to a mistake oracle $\mathsf{O}_{\calU}$ and an online learner $\calA$ for $\calC$ with sample complexity, number of calls to $\calA$, and number of calls to $\mathsf{O}_{\calU}$ that is at most $2\frac{M_\calA}{\eps} \log \inparen{\frac{M_\calA}{\delta}}$, where $M_\calA$ is the mistake-bound of online learner $\calA$.
\end{theorem}

\begin{proof}[Proof sketch]
{\vskip -2mm}
Run the online learner $\calA$ on the sequence of input examples using the mistake oracle $\mathsf{O}_{\calU}$ to find mistakes. Details and algorithm are provided in Appendix~\ref{sec:appendix}.
\end{proof}

\begin{example}
Let $\calC$ be the class of OR functions over the boolean hypercube $\set{0,1}^n$. There is an online learner $\calA$ that learns $\calC$ with a mistake bound $M_\calA = n$. Theorem~\ref{thm:online}  implies that we can robustly learn $\calC$ using $\calA$ with sample complexity, number of calls to $\calA$, and number of calls to $\mathsf{O}_\calU$ that is at most $2\frac{n}{\eps} \log \inparen{\frac{n}{\delta}}$.
\end{example}
\section{Discussion}

The main contribution of this paper is in formulating the question of reducing adversarially robust learning to standard non-robust learning and providing answers in some settings. We outline a few directions for furture work below. 

\textbf{Mistake Oracle for $\calU$} This is a more challenging setting (but perhaps more realistic) where the reduction algorithm has no knowledge of $\calU$ and can only interact with a mistake oracle for $\calU$. Theorem~\ref{thm:online} shows that online learnability is sufficient for robust learning in this model. Beyond this, are there weaker conditions that would enable robust learning under this model? Or is having an online learner essential? What if we consider specific target classes? \citet{icml2020_6130} recently gave an algorithm that robustly learns halfspaces in this model. A natural next step is to ask which other classes can be robustly learned in this model, or more ambitiously characterize a necessary and sufficient condition for learning in this model.

\textbf{Agnostic Setting} We focused only on robust PAC learning in the realizable setting, where we assume there is a $c\in\calC$ with zero robust error.  It would be desirable to extend our results also to  the agnostic setting, where we want to compete with the best $c\in\calC$. We remark that an agnostic-to-realizable reduction described in \citet[Theorem 6]{pmlr-v99-montasser19a} can be used in our setting, however, it has runtime that is exponential in $\vc(\calA)$. Another attempt through the agnostic boosting frameworks \citep[e.g.][]{DBLP:conf/nips/KalaiK09} requires a non-robust PAC learner $\calA$ with error $\eps$ that scales with $|\calU|^2$, which results in a sample complexity that depends on $|\calU|$, and this is something we would like to avoid. 

\textbf{Boosting and Robustness} Boosting has led to many exciting developments in theory and practice of machine learning. It started with asking: Can we boost the accuracy of weak predictors to attain a predictor with high accuracy? \cite{DBLP:journals/jcss/FreundS97} showed that boosting the accuracy is possible and can be done efficiently. What we consider in this paper can be viewed as a question of boosting robustness: Can we boost non-robust predictors to attain a {\em robust} predictor? and can we do this efficiently? Another natural question to consider which we did not study in this paper is: Can we boost \textit{weakly} robust predictors to attain a {\em robust} predictor?

\section*{Broader Impact}
Learning predictors that are robust to adversarial perturbations is an important challenge in contemporary machine learning. Current machine learning systems have been shown to be brittle against different notions of robustness such as adversarial perturbations \citep{szegedy2013intriguing,biggio2013evasion,goodfellow2014explaining}, and there is an ongoing effort to devise methods for learning predictors that {\em are} adversarially robust. As machine learning systems become increasingly integrated into our everyday lives, it becomes crucial to provide guarantees about their performance, even when they are used outside their intended conditions.

We already have many tools developed for standard learning, and having a universal \textit{wrapper} that can take any standard learning method and turn it into a \textit{robust} learning method could greatly simplify the development and deployment of learning that is \textit{robust} to test-time adversarial perturbations. The results that we present in this paper are still mostly theoretical, and limited to the realizable setting, but we expect and hope they will lead to further theoretical study as well as practical methodological development with direct impact on applications. 

In this work we do not deal with training-time adversarial attacks, which is a major, though very different, concern in many cases.

As with any technology, having a more robust technology can have positive and negative societal consequences, and this depends mainly on how such technology is utilized. Our intent from this research is to help with the design of robust machine learning systems for application domains such as healthcare and transportation where its critical to ensure performance guarantees even outside intended conditions. In situations where there is a tradeoff between robustness and accuracy, this work might be harmful in that it would prioritize robustness over accuracy and this may not be ideal in some application domains. 

%Authors are required to include a statement of the broader impact of their work, including its ethical aspects and future societal consequences. 
%Authors should discuss both positive and negative outcomes, if any. For instance, authors should discuss a) 
%who may benefit from this research, b) who may be put at disadvantage from this research, c) what are the consequences of failure of the system, and d) whether the task/method leverages
%biases in the data. If authors believe this is not applicable to them, authors can simply state this.

%Use unnumbered first level headings for this section, which should go at the end of the paper. {\bf Note that this section does not count towards the eight pages of content that are allowed.}

\begin{ack}
We would like to thank Shay Moran for the insightful discussions that led to the formalization of the question we study in this paper. We also thank the anonymous reviewers for their thoughtful and helpful feedback. This work is partially supported by DARPA\footnote{This paper does not reflect the position or the policy of the Government, and no endorsement should be inferred.} cooperative agreement HR00112020003.
%Use unnumbered first level headings for the acknowledgments. All acknowledgments
%go at the end of the paper before the list of references. Moreover, you are required to declare 
%unding (financial activities supporting the submitted work) and competing interests (related financial activities outside the submitted work). 
%More information about this disclosure can be found at: \url{https://neurips.cc/Conferences/2020/PaperInformation/FundingDisclosure}.

%Do {\bf not} include this section in the anonymized submission, only in the final paper. You can use the \texttt{ack} environment provided in the style file to autmoatically hide this section in the anonymized submission.
\end{ack}

\bibliographystyle{plainnat}
\bibliography{learning}

\begin{thebibliography}{24}
\providecommand{\natexlab}[1]{#1}
\providecommand{\url}[1]{\texttt{#1}}
\expandafter\ifx\csname urlstyle\endcsname\relax
  \providecommand{\doi}[1]{doi: #1}\else
  \providecommand{\doi}{doi: \begingroup \urlstyle{rm}\Url}\fi

\bibitem[Assouad(1983)]{assouad:83}
P.~Assouad.
\newblock Densit\'e et dimension.
\newblock \emph{Annales de l'Institut Fourier (Grenoble)}, 33\penalty0
  (3):\penalty0 233--282, 1983.

\bibitem[Attias et~al.(2019)Attias, Kontorovich, and
  Mansour]{DBLP:conf/alt/AttiasKM19}
Idan Attias, Aryeh Kontorovich, and Yishay Mansour.
\newblock Improved generalization bounds for robust learning.
\newblock In Aur{\'{e}}lien Garivier and Satyen Kale, editors,
  \emph{Algorithmic Learning Theory, {ALT} 2019, 22-24 March 2019, Chicago,
  Illinois, {USA}}, volume~98 of \emph{Proceedings of Machine Learning
  Research}, pages 162--183. {PMLR}, 2019.
\newblock URL \url{http://proceedings.mlr.press/v98/attias19a.html}.

\bibitem[Balcan(2010)]{nina}
Maria-Florina Balcan.
\newblock Lecture notes - machine learning theory, January 2010.

\bibitem[Biggio et~al.(2013)Biggio, Corona, Maiorca, Nelson, {\v{S}}rndi{\'c},
  Laskov, Giacinto, and Roli]{biggio2013evasion}
Battista Biggio, Igino Corona, Davide Maiorca, Blaine Nelson, Nedim
  {\v{S}}rndi{\'c}, Pavel Laskov, Giorgio Giacinto, and Fabio Roli.
\newblock Evasion attacks against machine learning at test time.
\newblock In \emph{Joint European conference on machine learning and knowledge
  discovery in databases}, pages 387--402. Springer, 2013.

\bibitem[Blumer et~al.(1989)Blumer, Ehrenfeucht, Haussler, and
  Warmuth]{blumer:89}
A.~Blumer, A.~Ehrenfeucht, D.~Haussler, and M.~Warmuth.
\newblock Learnability and the {Vapnik-Chervonenkis} dimension.
\newblock \emph{Journal of the Association for Computing Machinery},
  36\penalty0 (4):\penalty0 929--965, 1989.

\bibitem[Brown()]{brown2011wasted}
Daniel~G Brown.
\newblock How {I} wasted too long finding a concentration inequality for sums
  of geometric variables.

\bibitem[Bubeck et~al.(2019)Bubeck, Lee, Price, and
  Razenshteyn]{bubeck2019adversarial}
Sebastien Bubeck, Yin~Tat Lee, Eric Price, and Ilya Razenshteyn.
\newblock Adversarial examples from computational constraints.
\newblock In \emph{International Conference on Machine Learning}, pages
  831--840, 2019.

\bibitem[Ehrenfeucht et~al.(1989)Ehrenfeucht, Haussler, Kearns, and
  Valiant]{ehrenfeucht:89}
A.~Ehrenfeucht, D.~Haussler, M.~Kearns, and L.~Valiant.
\newblock A general lower bound on the number of examples needed for learning.
\newblock \emph{Information and Computation}, 82\penalty0 (3):\penalty0
  247--261, 1989.

\bibitem[Feige et~al.(2015)Feige, Mansour, and
  Schapire]{DBLP:conf/colt/FeigeMS15}
Uriel Feige, Yishay Mansour, and Robert~E. Schapire.
\newblock Learning and inference in the presence of corrupted inputs.
\newblock In Peter Gr{\"{u}}nwald, Elad Hazan, and Satyen Kale, editors,
  \emph{Proceedings of The 28th Conference on Learning Theory, {COLT} 2015,
  Paris, France, July 3-6, 2015}, volume~40 of \emph{{JMLR} Workshop and
  Conference Proceedings}, pages 637--657. JMLR.org, 2015.
\newblock URL \url{http://proceedings.mlr.press/v40/Feige15.html}.

\bibitem[Feige et~al.(2018)Feige, Mansour, and
  Schapire]{DBLP:conf/alt/FeigeMS18}
Uriel Feige, Yishay Mansour, and Robert~E. Schapire.
\newblock Robust inference for multiclass classification.
\newblock In Firdaus Janoos, Mehryar Mohri, and Karthik Sridharan, editors,
  \emph{Algorithmic Learning Theory, {ALT} 2018, 7-9 April 2018, Lanzarote,
  Canary Islands, Spain}, volume~83 of \emph{Proceedings of Machine Learning
  Research}, pages 368--386. {PMLR}, 2018.
\newblock URL \url{http://proceedings.mlr.press/v83/feige18a.html}.

\bibitem[Freund and Schapire(1997)]{DBLP:journals/jcss/FreundS97}
Yoav Freund and Robert~E. Schapire.
\newblock A decision-theoretic generalization of on-line learning and an
  application to boosting.
\newblock \emph{J. Comput. Syst. Sci.}, 55\penalty0 (1):\penalty0 119--139,
  1997.
\newblock \doi{10.1006/jcss.1997.1504}.
\newblock URL \url{https://doi.org/10.1006/jcss.1997.1504}.

\bibitem[Goodfellow et~al.(2014)Goodfellow, Shlens, and
  Szegedy]{goodfellow2014explaining}
Ian~J Goodfellow, Jonathon Shlens, and Christian Szegedy.
\newblock Explaining and harnessing adversarial examples.
\newblock \emph{arXiv preprint arXiv:1412.6572}, 2014.

\bibitem[Goodfellow et~al.(2015)Goodfellow, Shlens, and
  Szegedy]{DBLP:journals/corr/GoodfellowSS14}
Ian~J. Goodfellow, Jonathon Shlens, and Christian Szegedy.
\newblock Explaining and harnessing adversarial examples.
\newblock In Yoshua Bengio and Yann LeCun, editors, \emph{3rd International
  Conference on Learning Representations, {ICLR} 2015, San Diego, CA, USA, May
  7-9, 2015, Conference Track Proceedings}, 2015.
\newblock URL \url{http://arxiv.org/abs/1412.6572}.

\bibitem[Kalai and Kanade(2009)]{DBLP:conf/nips/KalaiK09}
Adam Kalai and Varun Kanade.
\newblock Potential-based agnostic boosting.
\newblock In Yoshua Bengio, Dale Schuurmans, John~D. Lafferty, Christopher
  K.~I. Williams, and Aron Culotta, editors, \emph{Advances in Neural
  Information Processing Systems 22: 23rd Annual Conference on Neural
  Information Processing Systems 2009. Proceedings of a meeting held 7-10
  December 2009, Vancouver, British Columbia, Canada}, pages 880--888. Curran
  Associates, Inc., 2009.
\newblock URL
  \url{http://papers.nips.cc/paper/3676-potential-based-agnostic-boosting}.

\bibitem[Madry et~al.(2018)Madry, Makelov, Schmidt, Tsipras, and
  Vladu]{DBLP:conf/iclr/MadryMSTV18}
Aleksander Madry, Aleksandar Makelov, Ludwig Schmidt, Dimitris Tsipras, and
  Adrian Vladu.
\newblock Towards deep learning models resistant to adversarial attacks.
\newblock In \emph{6th International Conference on Learning Representations,
  {ICLR} 2018, Vancouver, BC, Canada, April 30 - May 3, 2018, Conference Track
  Proceedings}. OpenReview.net, 2018.
\newblock URL \url{https://openreview.net/forum?id=rJzIBfZAb}.

\bibitem[Mansour et~al.(2015)Mansour, Rubinstein, and
  Tennenholtz]{DBLP:conf/soda/MansourRT15}
Yishay Mansour, Aviad Rubinstein, and Moshe Tennenholtz.
\newblock Robust probabilistic inference.
\newblock In Piotr Indyk, editor, \emph{Proceedings of the Twenty-Sixth Annual
  {ACM-SIAM} Symposium on Discrete Algorithms, {SODA} 2015, San Diego, CA, USA,
  January 4-6, 2015}, pages 449--460. {SIAM}, 2015.
\newblock \doi{10.1137/1.9781611973730.31}.
\newblock URL \url{https://doi.org/10.1137/1.9781611973730.31}.

\bibitem[Montasser et~al.(2019)Montasser, Hanneke, and
  Srebro]{pmlr-v99-montasser19a}
Omar Montasser, Steve Hanneke, and Nathan Srebro.
\newblock Vc classes are adversarially robustly learnable, but only improperly.
\newblock In Alina Beygelzimer and Daniel Hsu, editors, \emph{Proceedings of
  the Thirty-Second Conference on Learning Theory}, volume~99 of
  \emph{Proceedings of Machine Learning Research}, pages 2512--2530, Phoenix,
  USA, 25--28 Jun 2019. PMLR.

\bibitem[Montasser et~al.(2020)Montasser, Goel, Diakonikolas, and
  Srebro]{icml2020_6130}
Omar Montasser, Surbhi Goel, Ilias Diakonikolas, and Nati Srebro.
\newblock Efficiently learning adversarially robust halfspaces with noise.
\newblock In \emph{Proceedings of Machine Learning and Systems 2020}, pages
  10630--10641. 2020.

\bibitem[Moran and Yehudayoff(2016)]{moran:16}
S.~Moran and A.~Yehudayoff.
\newblock Sample compression schemes for {VC} classes.
\newblock \emph{Journal of the {ACM}}, 63\penalty0 (3):\penalty0 21:1--21:10,
  2016.

\bibitem[Salman et~al.(2020)Salman, Sun, Yang, Kapoor, and
  Kolter]{salman2020black}
Hadi Salman, Mingjie Sun, Greg Yang, Ashish Kapoor, and J~Zico Kolter.
\newblock Black-box smoothing: A provable defense for pretrained classifiers.
\newblock \emph{arXiv preprint arXiv:2003.01908}, 2020.

\bibitem[Schapire and Freund(2012)]{schapire:12}
R.~E. Schapire and Y.~Freund.
\newblock \emph{Boosting}.
\newblock Adaptive Computation and Machine Learning. MIT Press, Cambridge, MA,
  2012.

\bibitem[Szegedy et~al.(2013)Szegedy, Zaremba, Sutskever, Bruna, Erhan,
  Goodfellow, and Fergus]{szegedy2013intriguing}
Christian Szegedy, Wojciech Zaremba, Ilya Sutskever, Joan Bruna, Dumitru Erhan,
  Ian Goodfellow, and Rob Fergus.
\newblock Intriguing properties of neural networks.
\newblock \emph{arXiv preprint arXiv:1312.6199}, 2013.

\bibitem[Vapnik and Chervonenkis(1971)]{vapnik:71}
V.~Vapnik and A.~Chervonenkis.
\newblock On the uniform convergence of relative frequencies of events to their
  probabilities.
\newblock \emph{Theory of Probability and its Applications}, 16\penalty0
  (2):\penalty0 264--280, 1971.

\bibitem[Vapnik and Chervonenkis(1974)]{vapnik:74}
V.~Vapnik and A.~Chervonenkis.
\newblock \emph{Theory of Pattern Recognition}.
\newblock Nauka, Moscow, 1974.

\end{thebibliography}

\newpage
\appendix 
\section{Appendix}
\label{sec:appendix}

\begin{proof}[Proof of Lemma~\ref{lem:dualvc-convexhull}]
Consider a {\em dual space} $\Bar{\calG}$: a set of functions $\Bar{g}_x: \co^{k}(\calH) \to \calY$ defined as $\Bar{g}_x(f)=f(x)$ for each $f=\MAJ(h_1,\dots,h_k) \in \co^k(\calH)$ and each $x\in \calX$. It follows by definition of dual VC dimension that $\vc(\Bar{\calG})=\vc^*(\co^k(\calH))$. Similarly, define another dual space $\calG$: a set of functions $g:\calH \to \calY$ defined as $g(x)=h(x)$ for each $h\in \calH$ and each $x\in \calX$. We know that $\vc(\calG)=\vc^*(\calH)=d^*$. Observe that by definition of $\calG$ and $\Bar{\calG}$, we have that for each $x\in\calX$ and each $f=\MAJ(h_1,\dots,h_k)\in\co^k(\calH)$,
\[\Bar{g}_x(f)=f(x)=\MAJ(h_1,\dots,h_k)(x)=\sign\inparen{\sum_{i=1}^k h_i(x)}=\sign\inparen{\sum_{i=1}^{k} g_x(h_i)}.\]

By the Sauer-Shelah Lemma applied to dual class $\calG$, for any set $H=\set{h_1,\dots,h_n} \subseteq \calH$, the number of possible behaviors
\begin{equation}
\label{eqn:sauer-dual}
    |\calG|_H|:= | \set{ (g_x(h_1), \dots, g_x(h_n)): x\in \calX } | \leq {n \choose \leq d^*}.
\end{equation}
Consider a set $F=\set{f_1,\dots,f_m}\subseteq \co^{k}(\calH)$, the number of possible behaviors can be upperbounded as follows:
\begin{align*}
   \inabs{\Bar{\calG}|_F} &= \inabs{\set{ (\Bar{g}_x(f_1), \dots, \Bar{g}_x(f_m)): x\in \calX }}\\
   &= \inabs{\set{(\Bar{g}_x(\MAJ(h^1_1,\dots,h^k_1)), \dots, \Bar{g}_x(\MAJ(h^1_m,\dots,h^k_m))): x\in \calX}} \\
   &= \inabs{ \set{\inparen{\sign\inparen{\sum_{i=1}^{k} g_x(h_i)}, \dots, \sign\inparen{\sum_{i=1}^{k} g_x(h_i)}}: x\in\calX} }\\
   &\stackrel{(i)}{\leq} \inabs{ \set{(g_x(h_1^1), \dots, g_x(h_1^k), g_x(h_2^1),\dots, g_x(h_2^k),\dots,g_x(h_m^1),\dots,g_x(h_m^k)): x\in \calX} }\\
   &\stackrel{(ii)}{\leq} {mk \choose \leq d^*},
\end{align*}
where $(i)$ follows from observing that each expanded vector $(g_x(h_i^1),\dots,g_x(h_i^k))_{i=1}^{m} \in \calY^{mk}$ can map to at most one vector $\inparen{\sign\inparen{\sum_{i=1}^{k} g_x(h_i)}, \dots, \sign\inparen{\sum_{i=1}^{k} g_x(h_i)}}\in \calY^m$, and $(ii)$ follows from Equation~\ref{eqn:sauer-dual}. Observe that if $\inabs{\Bar{\calG}|_F} < 2^m$, then by definition, $F$ is not shattered by $\Bar{\calG}$, and this implies that $\vc(\Bar{\calG}) < m$. Thus, to conclude the proof, we need to find the smallest $m$ such that ${mk \choose \leq d^*} < 2^m$. It suffices to check that $m=O(d^* \log k)$ satisfies this condition.
\end{proof}

\begin{lem} [\cite{pmlr-v99-montasser19a}]
\label{lem:robust-compression}
For any $k \in \bbN$ and fixed function $\phi : (\calX \times \calY)^{k} \to \calY^{\calX}$, for any distribution $P$ over $\calX \times \calY$ and any $m \in \bbN$, 
for $S = \{(x_{1},y_{1}),\ldots,(x_{m},y_{m})\}$ iid $P$-distributed random variables,
with probability at least $1-\delta$, 
if $\exists i_{1},\ldots,i_{k} \in \{1,\ldots,m\}$ 
s.t.\ $\hat{R}_{\calU}(\phi((x_{i_{1}},y_{i_{1}}),\ldots,(x_{i_{k}},y_{i_{k}}));S) = 0$, 
then 
\begin{equation*}
R_{\calU}(\phi((x_{i_{1}},y_{i_{1}}),\ldots,(x_{i_{k}},y_{i_{k}}));P) \leq \frac{1}{m-k} (k\ln(m) + \ln(1/\delta)).
\end{equation*}
\end{lem}

\begin{proof} [Proof sketch of Claim~\ref{claim:lowerbound-mistakeoracle}]
Let $\calB$ be an arbitrary reduction algorithm. Let $x_0,x_1 \in \calX$, and $k\in \bbN$. Pick arbitrary points $Z=\set{z_1,\dots,z_{2k}} \subseteq \calX$. Let $X=\set{x_0,x_1}\cup Z$. Let $b\in \set{0,1}^{2k}$ be a bit string drawn uniformly at random from the set $\set{b\in \set{0,1}^{2k}: \sum_i b_i=k}$, think of this as a random partition of $Z$ into two equal sets $Z_0$ and $Z_1$. For each $y\in \set{0,1}$, define $\calU_b (x_y)$ to include $x_y$ and all perturbations $z\in Z_y$. Also, foreach $z\in Z$ define $\calU_b(z)=\set{z}$. Similarly, define target class $\calC_b$ to include only a single hypothesis $c_b$ where $c_b(\calU(x_0)) = 0$ and $c_b(\calU(x_1))= 1$. We will consider an $\ERM$ that uses the set of thresholds $\calH_{\phi}=\set{x \mapsto \ind[\phi(x) \geq \theta]: \theta \in \bbR}$ as its hypothesis class, where $\phi$ is a random embedding such that for each $z_0 \in \calU_b(x_0)$ and each $z_1 \in \calU_b(x_1)$: $\phi(z_0) < \phi(z_1)$; this guarantees that the random hypothesis $c_b$ is realized by some $h\in \calH_\phi$. On any input $L \subseteq X \times \set{0,1}$, we define the $\ERM$ to return the earliest possible threshold that reveals as few $0$'s as possible. 

Since algorithm $\calB$ only sees training data $S=\set{(x_0, 0), (x_1, 1)}$ as its input, by picking $b$ uniformly at random, $\calB$ has no way of knowing which perturbations belong to $\calU(x_0)$ and which belong to $\calU(x_1)$, and therefore its forced to call the mistake oracle $\mathsf{O}_{\calU}$ at least $k$ times. The $\ERM$ oracle is designed such that it will reveal as little information about this as possible.

Suppose that we run algorithm $\calB$ for $T$ rounds, where in each round $t\leq T$, $\calB$ maintains a predictor $f_t:X \to \set{0, 1}$ that determines that labeling of $x_0,x_1$ and the set of perturbations $Z$. We will show that, in expectation over the random choice of $b$ and $\phi$, in order for the final predictor $f_T$ outputted by $\calB$ to have robust loss zero on $S$, i.e. $\Risk_{\calU_{b}} (f_T)=0$, the number of rounds $T$ needs to be at least $k$.

On each round $t\leq T$, $\calB$ is allowed to:
\begin{enumerate}
    \item Query the mistake oracle $\mathsf{O}_\calU$ with a query consisting of some predictor $g_t:X\to \set{0, 1}$ and a point $(x,y) \in X\times \set{0,1}$.
    \item Query the $\ERM$ oracle with a dataset $L_t\subseteq X\times \set{0,1}$.
\end{enumerate}

Let $M_t= \sum_{z\in Z} \ind[ f_t(z) \neq c_b(z)]$ be the number of mistakes at round $t$, and let $H_t = \set{ g_j,(x_j,y_j),L_j }_{j\leq t}$ denote the history of queries. Then, observe that 
\[\Ex_{b,\phi} \insquare{ M_t | M_{t-1}, H_{t-1}} \geq M_{t-1} - 1.\]
This is because oracle $\mathsf{O}_\calU$ reveals the ground truth label of at most 1 point at round $t$, and the $\ERM$ will move the threshold by at most one position. This implies that $\Ex_{b,\phi} [M_T | M_{0}, H_{0}] \geq M_{0} - T$. We can further condition on the event that $M_0 \geq k$ which has non-zero probability (since $b$ is picked uniformly at random). This implies, by the probabilistic method, that there exists $b,\phi$ such that for $T\leq k - 1$, $M_T \geq 1$. Therefore, by definition of $M_T$, $f_T$ is not be robustly correct on $S$ for $T\leq k-1$.
\end{proof}

\begin{proof}[Proof of Theorem~\ref{thm:online}]
Let $\calU$ be an arbitrary adversary and $\mathsf{O}_{\calU}$ its corresponding mistake oracle. Let $\calC \subseteq \calY^\calX$ be an arbitrary target class, and $\calA$ an online learner for $\calC$ with mistake bound $M_\calA < \infty$. We assume w.l.o.g.\ that the online learner $\calA$ is conservative, meaning that it does not update its state unless it makes a mistake. Algorithm~\ref{alg:robust-learner-mistakeO} in essence is a standard conversion of a learner in the mistake bound model to a learner in the PAC model (see e.g.\, \cite{nina}):

\begin{algorithm}[H]
\caption{Robust Learner with a Mistake Oracle.}\label{alg:robust-learner-mistakeO}
\SetKwInput{KwInput}{Input}                % Set the Input
\SetKwInput{KwOutput}{Output}              % set the Output
\SetKwFunction{FMain}{ZeroRobustLoss}
\DontPrintSemicolon
  \KwInput{$S=\set{(x_1,y_1),\dots, (x_m,y_m)}$, $\eps ,\delta$, black-box access to a an online learner $\calA$, black-box access to a mistake oracle $\mathsf{O}_{\calU}$}
  %\KwData{Testing set $x$}
  Initialize $h_0 = \calA(\emptyset)$.\;
  \For{$i \leq m$}{
    Certify the robustness of $h$ on $(x_i, y_i)$ by asking the mistake oracle $\mathsf{O}_{\calU}$.\;
    If $h_t$ is not robust on $(x_i, y_i)$, update $h_t$ by running $\calA$ on $(z,y_i)$, where $z$ is the perturbation returned by $\mathsf{O}_{\calU}$.\;
    Break when $h_t$ is robustly correct on a consecutive sequence
of length $\frac{1}{\eps}\log\inparen{\frac{M_\calA}{\delta}}$.\;
  }
\KwOutput{$h_t$.}
\end{algorithm}

\paragraph{Analysis} Let $\calD$ be an arbitrary distribution over $\calX\times \calY$ that is robustly realizable with some concept $c\in \calC$,i.e., $\Risk_\calU(c;\calD)=0$. Fix $\eps, \delta \in (0,1)$ and a sample size $m = 2\frac{M_\calA}{\eps} \log \inparen{\frac{M_\calA}{\delta}}$. 

Since online learner $\calA$ has a mistake bound of $M_\calA$, Algorithm~\ref{alg:robust-learner-mistakeO} will terminate in at most $\frac{M_\calA}{\eps} \log\inparen{\frac{M_\calA}{\delta}}$ steps of certification, which of course is an upperbound on the number of calls to the mistake oracle $\mathsf{O}_{\calU}$, and the number of calls to the online learner $\calA$. 

It remains to show that the output of Algorithm~\ref{alg:robust-learner-mistakeO}, the final predictor $h$, has low robust risk $\Risk_\calU(h;\calD)$. Throughout the runtime of Algorithm~\ref{alg:robust-learner-mistakeO}, the online learner can generate a sequence of at most $M_\calA+1$ predictors. There's the initial predictor from Step 1, plus the $M_\calA$ updated
predictors corresponding to potential updates by online learner $\calA$. Observe that the probability that the final $h$ has robust risk more than $\eps$ 
\[
    \Prob_{S\sim \calD^m} \!\!\insquare{\Risk_\calU(h;\calD) \!>\! \eps} \leq \!\!\Prob_{S\sim \calD^m} \!\insquare{ \exists j \!\in\! [M_\calA\!+\!1] \text{ s.t. } \Risk_\calU(h_j;\calD) \!>\! \eps}
    \!\leq\! (M_\calA\!+\!1) (1-\eps)^{\frac{1}{\eps} \log\inparen{\frac{M_{\!\calA}\!+1}{\delta}}}
    \!\leq\! \delta.
\]
Therefore, with probability at least $1-\delta$ over $S\sim \calD^m$, Algorithm~\ref{alg:robust-learner-mistakeO} outputs a predictor $h$ with robust risk $\Risk_\calU(h;\calD)\leq \eps$. Thus, Algorithm~\ref{alg:robust-learner-mistakeO} robustly PAC learns $\calC$ w.r.t.\ adversary $\calU$. 
\end{proof}

\end{document}